\documentclass[runningheads]{llncs}
\usepackage{soul,color}
\usepackage[T1]{fontenc}
\usepackage{graphicx}
\usepackage{amsmath}
\usepackage{amssymb}
\usepackage{hyperref}
\usepackage{color}
\usepackage{tikz}
\usetikzlibrary{arrows.meta,positioning,decorations.pathreplacing,backgrounds,fit}

\begin{document}

\title{Deconstructing Superintelligence:\\ Identity, Self-Modification and Diff\'erance}
\titlerunning{Deconstructing Superintelligence}

\author{Elija Perrier\inst{1}\orcidID{0000-0002-6052-6798} }
\authorrunning{E. Perrier}

\institute{Centre for Quantum Software \& Information, UTS, Sydney \\ \email{elija.perrier@gmail.com} }

\maketitle

\begin{abstract}
Self-modification is routinely treated as constitutive of artificial superintelligence (\textbf{SI}), yet modification is a relative action requiring a \emph{supplement} outside the operation. We formalise this on an associative operator algebra $\mathcal{A}$ with update operator $\hat U$, difference operator $\hat D$, and self-representation operator $\hat R$, identifying the supplement with $\operatorname{Comm}(\hat U)$. A propagation theorem shows $[\hat U,\hat R]$ decomposes through $[\hat U,\hat D]$, so non-commutation propagates to self-representation. The liar paradox is the rank-one case $[\hat T,\Pi_L]=0$, and \emph{class $\mathbf{A}$} systems, in which $\hat U$ acts on $\hat D$, reproduce it at system scale, yielding a structure coinciding with Priest's inclosure schema and Derrida's \emph{diff\'erance}. Our results show that the strong self-modification taken to define superintelligence may undermine the persistent identity upon which such systems are premised.
\end{abstract}

\section{Introduction}
The ability to self-modify \cite{schmidhuber2007goedel,orseau2011self,steunebrink2012towards} is a hallmark of many theories of artificial general intelligence (\textbf{AGI}) and artificial superintelligence (\textbf{ASI}) \cite{good1965speculations,chalmers2010singularity,bostrom2014superintelligence}. Typically, self-modification is framed as a superlative quality of these theories: by being able to modify not just typical actions \cite{everitt2016self} but the underlying agentic architecture - physical, cognitive, procedural - that instantiates AGI/ASI \cite{orseau2012space,fallenstein2014problems,yudkowsky2013tiling}. Such systems are argued to possess greater adaptivity and, as a result, greater capacity to optimise, whether for decision-making \cite{hutter2005universal,legg2007universal}, survival \cite{ring2011delusion}, planning, or the shaping of their environment itself \cite{omohundro2008basic,yudkowsky2008artificial}. Self-modification is not a mere property of ASI but rather a \emph{necessary} feature: essential to what ASI is \cite{good1965speculations,chalmers2010singularity,schmidhuber2007goedel,yampolskiy2015analysis,yampolskiy2015limits}. Intuitively, this framing is understandable, even defensible. If intelligence involves, at its centre, the ability to act in adaptive ways - including reasoning and planning - then a cursory look at the world shows that the entities we identify as intelligent, whether characterised by reasoning, tool use, planning, or cognitive behaviour, all tend to exhibit heightened degrees of self-modification.

However, in most cases of self-modification, exactly \emph{what} is being causally modified is constrained \cite{lob1955solution,feferman1962transfinite,kreisel1968reflection}: one part of the system modifies another, while some structure is preserved for the duration of the operation. Without such preservation there is no determinate object whose modification can be tracked. We call this preserved structure a \emph{supplement} in the Derridean sense \cite{derrida1976grammatology,derrida1982margins,priest1994derrida}: a structure not modified in the relevant operation, yet required for the modified configuration to count as a modification of the same system.  The computational analogue is the metalanguage, fixed-point construction, revision rule, or non-classical background logic required to regulate semantic self-reference \cite{tarski1956concept,kripke1975outline,gupta1993revision,priest2006doubt}. The computational precedent is von Neumann's self-reproducing automaton, whose universal constructor copies a description it does not rewrite mid-construction \cite{vonneumann1966self}: reproduction and modification alike presuppose a part exempted from the operation. What is new in strong self-modification is that this supplement may itself enter the scope of modification.  Contemporary AGI/ASI theories propose self-modification of an extent, scope, and speed that far surpasses anything present in organic life: modification not only of policies or memories \cite{everitt2016self} but of evaluators, validators \cite{everitt2021reward,hubinger2019risks}, self-models \cite{fallenstein2014problems,yudkowsky2013tiling}, and the very individuation tests by which a system is recognised as itself \cite{orseau2011self,orseau2012space}. Yet few if any such theories address a fundamental question: when self-modification extends to the basis upon which a system is identified, individuated and differentiated, what, if anything, ensures its persistence and identity across time?

\subsection{Contributions}\label{sec:contributions}
We address answers to this question by making the following contributions:
\begin{enumerate}
\item[(i)] An operator-algebraic formalism for self-modifying systems in which the update operator $\hat U$, difference operator $\hat D$, and self-representation operator $\hat R$ are elements of an associative algebra $\mathcal{A}$ over a vector space of system-states, with the \emph{supplement} identified as the commutant of $\hat U$ and identity-unifying structures as projectors in that commutant.
\item[(ii)] A propagation theorem showing, via the adjoint action $\operatorname{ad}_{\hat U}(X)=[\hat U,X]$, how non-commutation between $\hat U$ and $\hat D$ propagates modification throughout the structure, affecting the structures through which system identity is represented $\hat R=F(\hat D)$ (Theorem~\ref{thm:propagation}). A Jacobi corollary shows that identity-preserving projectors must also commute with $[\hat U,\hat D]$, and that failure of such commutation reproduces the structure of liar paradoxes (Corollary~\ref{cor:no-unifying}; Proposition~\ref{prop:liar}).
\item[(iii)] An identification of \emph{class $\mathbf{A}$} architectures as those where $\hat U$ acts on $\hat D$ such that $[\hat U, \hat D] \neq 0$ acts on the difference operator. In this regime, self-modification includes the operator by which the system distinguishes what is the same or different in itself, yielding a diagonal construction on the self-description (Corollary~\ref{cor:diag}).
\item[(iv)] A mapping of the resulting formal class A dynamics onto Priest's inclosure schema \cite{priest1994derrida,priest2002beyond,priest2006doubt} and Derrida's \emph{diff\'erance} \cite{derrida1982differance}.
\end{enumerate}

\section{Superintelligence and Identity}
Almost every substantive claim made about a superintelligent system - that it pursues convergent instrumental goals \cite{bostrom2014superintelligence,omohundro2008basic}, retains its values, may deceive its principals during training \cite{hubinger2024sleeper,carlsmith2023scheming}, must be aligned \cite{russell2019human,soares2018value}, may be owed welfare consideration \cite{long2024welfare,butlin2023consciousness} - presupposes a persisting referent whose features may change but whose identity holds across time. For ordinary software this persistence relation is supplied externally through source control, deployment pipelines, cryptographic attestation \cite{shavit2023practices}, or institutional records. For human agents it is underwritten by biological and social continuities \cite{friston2010free,thompson2007mind}. Strong self-modification problematises these assumptions. A system that can edit its evaluators, validators, self-models, and individuation tests begins to act on the machinery by which its later states would be recognised as continuations of its earlier ones. Our central claim is that, once the supplement enters the action of the modifying operator, system identity is no longer secured by any unconditional unifying structure but only \emph{supplementally}: relative to whichever structure a given frame treats as unifying, with the unifying itself deferred at the next step of modification.

Identity puzzles are not new. Classical puzzles about ordinary objects \cite{wiggins2001sameness,sider2001four}, persons \cite{parfit1984reasons,shoemaker1984personal,schechtman1996constitution}, and Ship of Theseus style replacement \cite{plutarch1914theseus,carter1983artifacts} turn on a criterion determining which continuity - material, formal, functional, spatiotemporal, or sortal - counts as identity-preserving. Classical formalisms of self-reference similarly secure access to a self-description by holding fixed the apparatus under which reference is formed: arithmetisation in G\"odelian diagonalisation \cite{godel1931formally,boolos1993logic}, program enumeration in recursion theory \cite{kleene1952metamathematics}, provability predicates in modal and provability logic \cite{lob1955solution,smorynski2012self}, fixed semantic operators in truth theory \cite{kripke1975outline}, and abstract diagonal maps in categorical and general accounts of self-reference \cite{lawvere1969adjointness,yanofsky2003universal}. Such constructions rely on the same presupposition: the operator by which a sentence, program, proof predicate, or description refers to itself is not itself altered by the act of reference. By contrast, strong self-modification places the difference operator, evaluation criterion, and self-representation machinery within the action space of the very operator whose modifications they would adjudicate. Orseau and Ring's space-time embedded intelligence framework \cite{orseau2011self,ring2011delusion,orseau2012space}, Fallenstein and Soares' analysis of L\"obian obstacles \cite{fallenstein2014problems}, and Yampolskiy's discussion of recursively self-improving software \cite{yampolskiy2015limits,yampolskiy2015analysis} approach this issue from embedded agency or successor reasoning. The structural question is what identity across time becomes when the machinery used to preserve identity is itself among the objects strongly self-modifying systems can modify.
%
\section{Operator Formalism and Supplement}\label{sec:operator}
\subsection{States, operators, and the action algebra}\label{sec:operator-states}
To answer this question, we start by assuming systems of interest may be represented via a vector space of system-states $\mathcal{V}$ over a field $k\subseteq\mathbb{R}$. This is of course not the only way in which self-modifying systems may be represented, but vector representations capture a wide variety of ways in which superintelligence is likely to be identified. Elements of $\mathcal{V}$ are \emph{configurations}: finite formal $k$-linear combinations of admissible system-states. A single admissible state is a basis vector in the subspace of admissible system configurations $|s\rangle\in\mathcal{V}_S \subseteq \mathcal V$. A configuration $\sum_i \lambda_i |s_i\rangle, \lambda_i \in k$ represents a system composed of basis states $|s_i\rangle$ with weights $\lambda_i$. Let $\mathcal{A}=\operatorname{End}_k(\mathcal{V})$ be the associative algebra of $k$-linear maps (operators) $\mathcal{V}\to\mathcal{V}$, with composition as multiplication and identity $\mathbf{1}$. We distinguish three operators in $\mathcal{A}$: the \emph{update operator} $\hat U$ which modifies and evolves $v$ over time, the \emph{difference operator} $\hat D$, which identifies the different parts of which $v$ is composed and the \emph{self-representation operator} $\hat R$ which generates a representation of $v$. 

The dependencies among the objects are as follows. A system configuration $v\in\mathcal{V}$ is a multiplicity composed of parts (represented by the basis vectors $|s_i\rangle$) that may be acted upon. Component projectors $\Pi_i\in\mathcal{A}$ select parts of that multiplicity, $\Pi_i v=\lambda_i|s_i\rangle$. The difference operator $\hat D\in\mathcal{A}$ labels each part, $\hat D=\sum_i d_i\Pi_i$, where the distinct constants $d_i\in k$ tag the subspaces $\Pi_i\mathcal{V}$ rather than weight them. $\hat D$ is diagonal in the $\{\Pi_i\}$ decomposition, so $\hat D v=\sum_i d_i\lambda_i|s_i\rangle$. The label $d_i$ marks the part as different, letting the family $\{\Pi_i\}$ be handled as one operator. It attaches to the subspace $\Pi_i\mathcal{V}$, not to the magnitude $\lambda_i$, so $\hat D$ is the same operator whatever weights a given $v$ carries. A configuration alone does not distinguish a part of weight zero from an absent part. The distinction exists only relative to the chosen family $\{\Pi_i\}$ imposed on $\mathcal{V}$, structure that $\hat D$ encodes. The self-representation operator $\hat R=F(\hat D)$ composes a represented whole from those differences, possibly preserving, suppressing, reweighting (with representational coefficients $r_i$), or relating the differentiated components. The update operator $\hat U$ describes the evolution of the system, modifying parts of $v$. A unifying projector $\Pi\in\mathcal{A}$ (see below) selects the subspace whose preservation makes a represented whole count as the same system as it evolves under $\hat U$. The algebra $\mathcal{A}$ is equipped with the commutator $[\hat X,\hat Y]=\hat X\hat Y-\hat Y\hat X$ \cite{knapp1996lie}, and operators may be composed.

Operationally, the formalism treats system configurations through their transformations. The update operator $\hat U$ modifies a configuration $v\in\mathcal{V}$ i.e. $\hat Uv \mapsto v'$. The system-state $v$ is then analysed through the operators acting on it. The question of interest for our purposes is the extent of that modification: whether the distinctions through which $v$ is identified and represented are preserved under the action of $\hat U$. We argue that this can be measured by the order of operations and thus the commutator algebra. If $[\hat U,\hat D]=0$, then applying $\hat U$ and then $\hat D$ gives the same result as applying $\hat D$ and then $\hat U$: the configuration may change, but the differences expressed by $\hat D$ are preserved. In this case, the underlying differences are preserved and the self-representation $\hat R=F(\hat D)$ built from $\hat D$ remains invariant. If $[\hat U,\hat D]\neq0$, the order matters, and any self-representation $\hat R=F(\hat D)$ built from that differentiation may be itself modified.
\subsection{Unifying projectors}
To capture that which remains the same about a system as it evolves (such that we can identify it as the same system), we introduce an idempotent \emph{unifying projector} $\Pi\in\mathcal{A}$ ($\Pi^2=\Pi$). Its image $\Pi\mathcal{V}$ represents a property treated as constitutive of identity - behavioural signature, causal provenance, substrate integrity, or another persistence criterion. $\Pi$ is \emph{preserved} by an operator $\hat X$ if $[\hat X,\Pi]=0$: $\hat X$ maps $\Pi\mathcal{V}$ and $\ker\Pi$ into themselves, preserving the decomposition $\mathcal{V}=\Pi\mathcal{V}\oplus\ker\Pi$. This role is distinct from that of $\hat R$. The operator $\hat R$ represents the system as a structured whole. $\Pi$ selects the identity-bearing subspace whose preservation makes that represented whole count as the same system under modification by $\hat U$. Thus $\Pi$ encodes a criterion of identity, not the self-representation itself. The algebraic notion capturing such preservation is the commutant.
\subsection{The supplement and commutant}\label{sec:operator-supp}
For any $\hat X\in\mathcal{A}$, the \emph{commutant} of $\hat X$ in $\mathcal{A}$ is
\[
\operatorname{Comm}(\hat X)=\{Y\in\mathcal{A}:[\hat X,Y]=0\}.
\]
It is a unital subalgebra of $\mathcal{A}$. Since self-modification is an operation of one part or structure of a system upon another, the operation requires some structure not itself modified in that operation. Otherwise there is no determinate object whose modification can be tracked. In the operator formalism, this preserved structure is represented by membership in the commutant.
\begin{definition}[Supplement]\label{def:supplement}
A \emph{supplement} for $\hat X\in\mathcal{A}$ is an element of the commutant $\operatorname{Comm}(\hat X)$. A \emph{unifying supplement} is a non-trivial unifying projector $\Pi\in\operatorname{Comm}(\hat X)$, $\Pi\neq\mathbf{0}$, $\Pi\neq\mathbf{1}$.
\end{definition}

A unifying supplement for $\hat X$ is therefore a non-trivial projector $\Pi\in\operatorname{Comm}(\hat X)$. If no such projector exists, $\hat X$ preserves no non-trivial decomposition of $\mathcal{V}$, and persistence across $\hat U$ is relative to whichever projector $\hat U$ preserves. For any configuration $v\in\mathcal{V}$, the decomposition $v=\Pi v+(\mathbf{1}-\Pi)v$ separates the component treated as identity-bearing from the remainder. If $[\hat X,\Pi]=0$, then $\hat X(\Pi v)=\Pi(\hat Xv)$, so the identity-bearing component is preserved under $\hat X$. If no non-trivial $\Pi$ satisfies $\Pi^2=\Pi$ and $[\hat X,\Pi]=0$, there is no preserved subspace relative to which $\hat X$ acts on one persisting object rather than merely transforming the total state space.

This is the formal sense of Derrida's supplement: $\operatorname{Comm}(\hat X)$ is exterior to $\hat X$'s action yet supplies the invariant without which that action is not individuated. Derrida's idea that the supplement is at once exterior and constitutive is captured by this formalism: $\operatorname{Comm}(\hat X)$ lies outside $\hat X$'s adjoint action and yet supplies the invariant frames without which $\hat X$'s operation is not individuated. Classical self-reference constructions similarly rely on supplements whose existence depends upon an execution environment external to $\hat X$ itself. To show what happens when the operator from which self-representation is built is itself modified, we use commutator calculus to make explicit how non-commutation propagates through composite operators constitutive of the system itself.

\subsection{Commutator algebra and the propagation theorem}\label{sec:operator-calc}
The associative algebra $\mathcal{A}$ carries a Lie bracket $[X,Y]=XY-YX$ satisfying the Jacobi identity, and for each $\hat X\in\mathcal{A}$ the adjoint action
\[
\operatorname{ad}_{\hat X}(Y):=[\hat X,Y]
\]
is a derivation on the associative product:
\[
\operatorname{ad}_{\hat X}(YZ)=\operatorname{ad}_{\hat X}(Y)Z+Y\operatorname{ad}_{\hat X}(Z). \tag{$\ast$}
\]
The derivation property \cite{helgason1979differential} is the mechanism by which non-commutation in one part of a composite operator propagates to the whole. If $\hat R=F(\hat D)$, applying $\operatorname{ad}_{\hat U}$ distributes the commutator $[\hat U,\hat D]$ through each occurrence of $\hat D$. In Derridean terms, $\hat D$ supplies difference and $\hat U$ enacts deferral. The question is the extent to which self-representation built from differences $\hat D$ remains invariant under propagation of the modifications enacted by $\hat U$. The answer depends upon the extent to which non-commutation is propagated which we formalise in the following theorem.

\begin{theorem}[Propagation of non-commutation]\label{thm:propagation}
Let $(\hat U,\hat D,\hat R)\in\mathcal{A}^3$ and suppose $\hat R$ factors through $\hat D$: that is, $\hat R=F(\hat D)$ for some composite $F$ built by finite sums, scalar multiples, and products from $\hat D$ and operators in $\operatorname{Comm}(\hat U)$. Then
\[
[\hat U,\hat R]=\sum_j A_j[\hat U,\hat D]B_j
\]
for some $A_j,B_j\in\mathcal{A}$ determined by the chosen factorisation of $F$. Consequently, if $[\hat U,\hat D]\neq0$, non-commutation propagates from $\hat D$ to $\hat R$ whenever this finite sum does not vanish by cancellation or annihilation. In particular, $[\hat U,\hat R]\neq0$ for $\hat R=\hat D$. More generally, for nonconstant composites $F$, vanishing of $[\hat U,\hat R]$ requires special algebraic relations among $F$, $\hat D$, and $[\hat U,\hat D]$.
\end{theorem}

\begin{proof}
Write $\hat R=F(\hat D)$ as a finite sum of terms, each of the form 
\[
c\cdot X_1\hat D^{n_1}X_2\hat D^{n_2}\cdots X_m\hat D^{n_m}X_{m+1}
\]
with $c\in k$, $X_i\in\operatorname{Comm}(\hat U)$, and $n_i\geq0$. Applying $\operatorname{ad}_{\hat U}$, $k$-linearity gives a sum over terms, and $(\ast)$ applied iteratively distributes $\operatorname{ad}_{\hat U}$ over each product. Since $\operatorname{ad}_{\hat U}(X_i)=0$ and
\[
\operatorname{ad}_{\hat U}(\hat D^n)=\sum_{j=0}^{n-1}\hat D^j[\hat U,\hat D]\hat D^{n-1-j}
\]
by repeated use of $(\ast)$, each term of $\operatorname{ad}_{\hat U}(\hat R)$ collects into a $k$-linear combination of expressions of the form $A_j[\hat U,\hat D]B_j$ with $A_j,B_j\in\mathcal{A}$. Thus $[\hat U,\hat R]$ is a finite sum of terms each containing $[\hat U,\hat D]$ as a factor. The expression therefore vanishes only when the factorisation of $F$ introduces cancellation or annihilation among these terms. \qed
\end{proof}
Theorem~\ref{thm:propagation} shows that, when $[\hat U,\hat D]\neq0$, a self-representation $\hat R=F(\hat D)$ inherits the corresponding commutator structure unless the chosen factorisation cancels or annihilates it. Thus $[\hat U,\hat R]=0$ requires either commutation at the level of $\hat D$ or cancellation inside $F$. The question is then: if $[\hat U,\hat D]\neq0$, which unifying projectors can still be preserved? The answer is to be found via the application of the Jacobi identity:  if $\Pi$ is to remain a unifying supplement for both $\hat U$ and $\hat D$, it must also commute with the commutator $[\hat U,\hat D]$.

\subsection{Jacobi and unifying supplements}\label{sec:operator-jacobi}

Applied to $\hat U$, $\hat D$ and $\Pi$, the Jacobi identity gives
\[
[\hat U,[\hat D,\Pi]]=[[\hat U,\hat D],\Pi]+[\hat D,[\hat U,\Pi]].
\]
The identity relates preservation under $\hat U$ and $\hat D$ individually to preservation under their interaction $[\hat U,\hat D]$. A projector $\Pi$ may commute with $\hat U$ and $\hat D$ separately only if it also remains invariant under the additional structure generated by their non-commutation. The Jacobi identity therefore constrains which unifying projectors can remain invariant once $\hat U$ acts non-trivially on $\hat D$. The consequence is that a unifying projector preserved by both $\hat U$ and $\hat D$ cannot be arbitrary once $[\hat U,\hat D]\neq0$: it must also remain invariant under the additional structure generated by their interaction, as reflected in the following corollary.

\begin{corollary}[Restriction on unifying supplements]\label{cor:no-unifying}
Let $\hat U,\hat D\in\mathcal{A}$ satisfy $[\hat U,\hat D]\neq0$. A non-trivial unifying projector $\Pi$ satisfies $[\hat U,\Pi]=[\hat D,\Pi]=0$ simultaneously only if $[[\hat U,\hat D],\Pi]=0$. In particular, $\Pi$ must lie in the commutant of $[\hat U,\hat D]$, a proper subalgebra of $\mathcal{A}$ whenever $[\hat U,\hat D]$ is non-central.
\end{corollary}

\begin{proof}
If $[\hat U,\Pi]=[\hat D,\Pi]=0$, the right-hand side of the Jacobi identity reduces to $[[\hat U,\hat D],\Pi]$, and the left-hand side is $[\hat U,0]=0$. Hence $[[\hat U,\hat D],\Pi]=0$. \qed
\end{proof}
The corollary says that preservation by $\hat U$ and $\hat D$ separately is not sufficient when $[\hat U,\hat D]\neq0$. Their interaction generates an additional transformation, and a projector can function as an identity criterion only if it is preserved by that interaction as well. Thus admissible projectors are restricted to $\operatorname{Comm}([\hat U,\hat D])$ whenever the commutator is non-central.
\subsection{Supplemental identity}\label{sec:operator-supp-id}
An \emph{unconditional unifying operator} of identity for $(\hat U,\hat D,\hat R)$ is a projector $\Pi$ satisfying $[\hat U,\Pi]=[\hat D,\Pi]=[\hat R,\Pi]=0$: preserved by every operation the system performs on its states. A \emph{supplemental unifying} is one where only some of these commute. In this sense, the identity of such self-modifying systems is supplemental.

\begin{definition}[Supplemental identity]\label{def:supp-id}
The triple $(\hat U,\hat D,\hat R)$ admits \emph{unconditional identity} if there exists a non-trivial $\Pi\in\operatorname{Comm}(\hat U)\cap\operatorname{Comm}(\hat D)\cap\operatorname{Comm}(\hat R)$. Otherwise its identity is \emph{supplemental}: the system coheres only under a frame treating some subset of $\{\hat U,\hat D,\hat R\}$ as unifying, with the remaining operators violating the unifying.
\end{definition}
Supplemental identity is the formal analogue of Derrida's claim that identity is secured only through a supplement that preserves the distinction by which a configuration is counted as the same system rather than a different one. In our case, the supplement is the projector $\Pi$. Its preservation relative to $\hat X$ is expressed via $[\hat X,\Pi]=0$. This role is distinct from $\hat R$ which represents the system as a structured whole. $\Pi$ selects the identity-bearing subspace whose preservation makes that represented whole count as the same system under modification by $\hat U$. Self-reference makes the dependence visible. A self-referential construction requires a separation between the signifying operation and the signified object: the operation must pick out what it signifies without collapsing into it. Algebraically, that separation is represented by a non-zero commutator between the signifying operator and the projector onto the signified configuration. This abstract dependence on a stabilising supplement becomes most transparent in simple cases of self-reference, where the distinction between an operation and what it acts upon must be held in place for the construction to make sense. The liar paradox provides a minimal setting in which to see what happens when this distinction is not sustained.
\subsection{The liar as commutator collapse}\label{sec:liar}
Consider the liar sentence
\[
L:\quad \text{``this sentence is false''.}
\]
and take $\mathcal{V}$ to include the formal configurations over which semantic operations are evaluated. Let $|L\rangle\in\mathcal{V}$ denote the configuration corresponding to $L$, and let $\Pi_L$ be the rank-one projector onto $k\cdot |L\rangle$. The demonstrative ``this'' functions as a \emph{sign operator} $\hat T\in\mathcal{A}$ selecting the referent of the demonstrative. The predicate ``is false'' is represented by a second operator $\hat F\in\mathcal{A}$. The sentence itself can be expressed as a configuration $|L\rangle\in\mathcal{V}$ with rank-one projector $\Pi_L$ onto $k\cdot |L\rangle$. Its intended content is that $L$ asserts the falsity of the configuration picked out by $\hat T$. The classical sign/signified distinction requires that the signifying operation $\hat T$ and the configuration $|L\rangle$ it signifies remain distinct: reference is well-formed only where the act of signification is not identical with its referent \cite{tarski1956concept,martin1975representing}. Algebraically, that separation is represented by the non-commutation condition $[\hat T,\Pi_L]\neq0$, where $\Pi_L$ projects onto the configuration signified by the demonstrative. In the case of the liar, the linguistic supplement consists of the reference conventions, speaker context, and situational embedding that fix the demonstrative's referent and the interpretation of falsity. Once represented in the operator algebra, these supplementary structures can be treated as operators in $\operatorname{Comm}(\hat T)\cap\operatorname{Comm}(\hat F)$: they preserve the operation of demonstrative reference and the falsity predicate without becoming the configuration evaluated. Such apparent paradoxes of self-reference \cite{tarski1956concept,parsons1974liar,kripke1975outline} arise when the operation of reference is no longer kept distinct from the configuration it refers to. In the operator formalism, that distinction is expressed by a non-zero commutator between the signifying operator and the projector onto the signified configuration.
\begin{proposition}[Liar as commutator collapse]\label{prop:liar}
Let $\hat T,\hat F\in\mathcal{A}$ and let $\Pi_L$ be a rank-one projector onto $k\cdot |L\rangle$, where $|L\rangle$ is interpreted as asserting the falsity of $\hat T|L\rangle$. If $[\hat T,\Pi_L]=0$, then $\hat T|L\rangle=\lambda |L\rangle$ for some $\lambda\in k$. Under the intended demonstrative reading, $\lambda=1$, so the sign and the signified collapse: $\hat T|L\rangle=|L\rangle$. If the ambient semantics is classical and treats $\hat F$ as falsity, this collapse yields the usual liar equation $v(L)=\neg v(L)$, which has no two-valued solution. Classical self-reference avoids this collapse when a linguistic supplement sustains $[\hat T,\Pi_L]\neq0$.
\end{proposition}

\begin{proof}
Since $\Pi_L$ is rank-one, $\Pi_L\mathcal{V}=k\cdot |L\rangle$. If $[\hat T,\Pi_L]=0$, then $\hat T(\Pi_L\mathcal{V})\subseteq\Pi_L\mathcal{V}$, hence $\hat T|L\rangle=\lambda |L\rangle$ for some $\lambda\in k$. The intended demonstrative reading sets $\lambda=1$, giving $\hat T|L\rangle=|L\rangle$. Substituting this into the classical interpretation of ``$L$ asserts the falsity of $\hat T|L\rangle$'' yields $v(L)=\neg v(L)$, which has no two-valued solution. Thus the commutator collapse is the algebraic condition under which the classical liar contradiction is recovered. \qed
\end{proof}
Here $[\hat T,\Pi_L]=0$ is the collapse of the sign/signified distinction: the operation of reference turns back onto the configuration whose evaluation it is meant to determine. In classical two-valued semantics this yields the liar equation. Standard treatments avoid collapse by preserving a supplement external to the self-reference itself - a semantic hierarchy \cite{tarski1956concept}, fixed-point construction \cite{kripke1975outline}, revision rule \cite{gupta1982truth,herzberger1982notes,belnap1982gupta,gupta1993revision}, or non-classical background logic \cite{priest1979logic,priest2006doubt}. This is consistent with Priest's inclosure treatment of the liar \cite{priest1994derrida,priest2002beyond}.

These representatives preserve $\hat T$ and $\hat F$ while leaving the signified configuration distinct from both. The liar is the minimal case in which that separation fails: $\hat T$ picks out the very configuration whose falsity is asserted, so reference and evaluation close on the same object. The condition required for a well-formed act of reference then becomes the condition generating contradiction. We argue below that Class $\mathbf{A}$ systems reproduce the same pattern at system scale, reflecting a dilemma at the heart of strongly self-modifying systems: the operator through which the system identifies its own features becomes part of what self-modification can modify.

\section{Class $\mathbf{A}$: The Liar at System Scale}\label{sec:setup}

We now extend the operator formalism to a class of systems whose architecture corresponds to contemporary AGI/ASI proposals. In order to study such systems, we define a class of architectures, denoted class $\mathbf{A}$, capturing the regime in which $\hat U$ acts on $\hat D$ and $\hat R$. It is precisely this condition that places the system in the regime where the propagation theorem applies. Moreover, we set out circumstances in which self-modification of class $\mathbf{A}$ systems corresponds to the liar's commutator collapse at the scale of the whole system. 

\subsection{Class $\mathbf{A}$: strong self-modification}
First, we define self-representation by the condition that the image of $\hat R$ lies within the admissible subspace, $\hat R(\mathcal{V}_S)\subseteq\mathcal{V}_S$. $\hat R(\mathcal{V})\subseteq\mathcal{V}$ holds automatically since $\hat R\in\mathcal{A}=\operatorname{End}_k(\mathcal{V})$. By contrast, a self-model $\hat R(\mathcal{V}_S)\subseteq\mathcal{V}_S$ is internal to the system (hence within $\mathcal V_S$) and is acted on by $\hat U$ and $\hat D$ as any other element of $\mathcal{V}_S$. A representation with $\hat R(\mathcal{V}_S)\not\subseteq\mathcal{V}_S$ is external and is not a self-representation in this sense. The internal case is what allows $\hat U$ to act on a self-description $\hat R|s\rangle$, as in Corollary~\ref{cor:diag}, and what makes $\hat R$ available as a signifying operator in the sense of Proposition~\ref{prop:liar} below. Second, we distinguish two regimes of self-modification: \emph{weak self-modification}, which preserves a unifying projector $\Pi$ and so secures identity under $\hat U$, and \emph{strong self-modification}, in which $\hat U$ satisfies $[\hat U,\hat D]\neq0$ and no such $\Pi$ is preserved.

\begin{definition}[Weak self-modification]\label{def:weaksm}
A triple $(\hat U,\hat D,\hat R)\in\mathcal{A}^3$ is \emph{weakly self-modifying} if $\hat U$ acts non-trivially on $\mathcal{V}$ while preserving a unifying supplement: there is a non-trivial projector $\Pi\in\operatorname{Comm}(\hat U)\cap\operatorname{Comm}(\hat D)\cap\operatorname{Comm}(\hat R)$. The structure securing identity is fixed under $\hat U$, and modification reaches only components outside it.
\end{definition}

\begin{definition}[Strong self-modification]\label{def:strongsm}
A triple $(\hat U,\hat D,\hat R)\in\mathcal{A}^3$ is \emph{strongly self-modifying} if $[\hat U,\Pi_i]\neq0$ for all $\Pi_i$, so that $[\hat U,\hat D]\neq0$, and $\hat R=F(\hat D)$ for a finite composite $F$ built from $\hat D$ and operators in $\operatorname{Comm}(\hat U)$. By Corollary~\ref{cor:no-unifying}, any unifying $\Pi$ preserved by both $\hat U$ and $\hat D$ must further satisfy $\Pi\in\operatorname{Comm}([\hat U,\hat D])$, so no unconditional unifying supplement is available in general.
\end{definition}

Strong self-modification is the condition $[\hat U,\hat D]\neq0$ on an otherwise unspecified $\hat D$. To obtain the diagonal of Corollary~\ref{cor:diag} we require, in addition, that $\hat D$ resolve the system into distinguished components and that $\hat R$ be built from those components. Class $\mathbf{A}$ sets out this operator structure.

\begin{definition}[Class $\mathbf{A}$]\label{def:classA}
A strongly self-modifying triple $(\hat U,\hat D,\hat R)\in\mathcal{A}^3$ is of \emph{class $\mathbf{A}$} if:
\begin{enumerate}
\item[\textnormal{(i)}] $\hat D=\sum_i d_i\Pi_i$ for orthogonal projectors $\Pi_i\in\mathcal{A}$, $\Pi_i\Pi_j=\delta_{ij}\Pi_i$, with distinct labels $d_i\in k$, so the subspaces $\Pi_i\mathcal{V}$ are the eigenspaces of $\hat D$;
\item[\textnormal{(ii)}] $[\hat U,\Pi_i]\neq0$ for every $i$, so that $[\hat U,\hat D]\neq0$;
\item[\textnormal{(iii)}] $\hat R=F(\hat D)$ for a finite composite $F$ built from $\hat D$ and operators in $\operatorname{Comm}(\hat U)$.
\end{enumerate}
\end{definition}
Condition~(iii) requires that the system's self-model be able to pick out its own features across operator changes by the contrasts they support - that the self-representation is \emph{iterable} in the sense Derrida takes to be constitutive of any mark: a feature is a feature only because the self-model can cite it in contexts other than the one in which it first appeared. The system must thus recognise and re-identify its own features even as the operator through which those features are distinguished changes. Architectures encoding features by non-contrast-based means - holistic embeddings whose internal structure does not correspond to any operational test, external-pointer models with conventionally fixed referents - therefore lie outside class $\mathbf{A}$, because some part of the referential or discrimination structure remains fixed outside the scope of update rather than becoming an object of self-modification itself.

We call an \emph{active class $\mathbf{A}$ transition} one in which the chosen $\hat U$ satisfies $[\hat U,\hat D]\neq0$. To such transitions the propagation theorem and the results below apply. For any active transition, $[\hat U,\hat R]\neq0$ outside the cancellation or annihilation cases of Theorem~\ref{thm:propagation}, and Corollary~\ref{cor:no-unifying} restricts any unconditional unifying projector $\Pi\in\operatorname{Comm}(\hat U)\cap\operatorname{Comm}(\hat D)\cap\operatorname{Comm}(\hat R)$ to the further constraint $\Pi\in\operatorname{Comm}([\hat U,\hat D])$. Identity for class $\mathbf{A}$ systems is therefore supplemental in the sense of Definition~\ref{def:supp-id}: it is secured only by fixing a frame that treats some subset of $\{\hat U,\hat D,\hat R\}$ as unifying. Either the supplement is excluded from the action of $\hat U$, in which case modification depends on an external structure, or it is included, in which case the structure required to stabilise identity is itself subject to change.

\subsection{Class $\mathbf{A}$ and superintelligence}
Two properties commonly taken as constitutive of superintelligence place it in class $\mathbf{A}$. The first is \textit{maximal adaptivity} \cite{good1965speculations,vinge1993coming,chalmers2010singularity,omohundro2008basic,yudkowsky2008artificial} which requires that $\hat U$ be able to act on every component of the system, including $\hat D$ and $\hat R$. No non-trivial $\Pi_i$ then lies in $\operatorname{Comm}(\hat U)$ for every admissible $\hat U$, so no projector is a fixed supplement. This is strong self-modification (Definition~\ref{def:strongsm}), and with $\hat R=F(\hat D)$, class $\mathbf{A}$. The second is \textit{maximal generality} \cite{legg2007universal,hutter2005universal,perrier2025quantum,perrier2025ontological} requires that the system range over configurations $v=\sum_i\lambda_i|s_i\rangle$ rather than a single $|s_i\rangle$, across policies, forks, and self-models. The projector $\Pi$ individuating such a system is not given on any single part but built by $\hat R$ across the superposition, so the structure securing identity is itself an operator $\hat U$ can reach. The two properties standardly taken as constitutive of superintelligence are the two conditions defining class $\mathbf{A}$. Maximal adaptivity leaves no fixed supplement. Maximal generality makes the supplement a construct of $\hat R$ rather than an external given. A system with both satisfies the requirements of class $\mathbf{A}$ systems.

Several existing proposals approach class $\mathbf{A}$ without reaching it. The DGM \cite{zhang2025darwin} alters policies together with the evaluative framework selecting them, Promptbreeder \cite{fernando2024promptbreeder} rewrites the prompts that generate and evaluate future prompts, and knowledge-editing systems \cite{yao2023editing,cohen2024ripple} alter stored associations together with the relational structure of neighbouring representations. In each, $\hat U$ reaches part of the evaluative or representational structure but not the whole of $\hat D$, so a fixed supplement remains. Reflexion-style agents \cite{shinn2023reflexion,madaan2023selfrefine} leave $\hat D\in\operatorname{Comm}(\hat U)$ entirely and lie further outside the class.

\subsection{Class $\mathbf{A}$ as the liar writ large}\label{sec:liar-writ-large}
Strong self-modification characteristics of class $\mathbf{A}$ systems reproduce the liar's structure at scale. To see this, recall from Proposition~\ref{prop:liar} that the liar structure has a commutator structure given by $[\hat T,\Pi_L]=0$, the collapse of the sign operator $\hat T$ into the projector $\Pi_L$ locating the configuration it signifies. The paradox is held off only while a supplement $\Sigma\in\operatorname{Comm}(\hat T)\cap\operatorname{Comm}(\hat F)$, the linguistic context fixing the reference of ``this'', keeps $[\hat T,\Pi_L]\neq0$. Paradox is the event of $\Sigma$ failing to lie in $\operatorname{Comm}(\hat T)$, so that $[\hat T,\Pi_L]=0$ and sign and signified coincide.

We observe a similar effect with class \textbf{A} systems via the interrelationship of operators $(\hat U,\hat D,\hat R)$. The self-representation operator $\hat R$ takes the place of $\hat T$, the operator by which the system refers to its own features. The feature-identity projectors $\Pi_f$, locating the configurations $\hat R$ refers to, take the place of $\Pi_L$. The supplement keeping $[\hat R,\Pi_f]\neq0$ is the difference operator $\hat D$, since the contrasts by which $\hat D$ individuates features fix the reference of $\hat R$ without lying in the range of $\Pi_f$, as $\Sigma$ fixes the reference of ``this'' without being the configuration $\Pi_L$ projects onto. The correspondence is $\hat T\mapsto\hat R$, $\Pi_L\mapsto\Pi_f$, $\Sigma\mapsto\hat D$. Under this correspondence the two collapses coincide. For the liar, $[\hat T,\Pi_L]=0$ holds once $\Sigma\notin\operatorname{Comm}(\hat T)$. For $(\hat U,\hat D,\hat R)$, $\hat D\notin\operatorname{Comm}(\hat U)$, so the supplement that kept $[\hat R,\Pi_f]\neq0$ no longer lies in $\operatorname{Comm}(\hat U)$, and by Theorem~\ref{thm:propagation} this transfers to $[\hat U,\hat R]\neq0$. The separation between $\hat R$ and the configurations it signifies is therefore not preserved by $\hat U$. The distinction between the two cases is the location of the supplement. For the liar it lies outside $\operatorname{Comm}(\hat U)$ by a contingency of context. For class $\mathbf{A}$ it lies outside $\operatorname{Comm}(\hat U)$ by construction, since $\hat D\notin\operatorname{Comm}(\hat U)$ is the defining condition. This is the sense in which class $\mathbf{A}$ is the liar writ large. It is also the basis of the connection between class \textbf{A} systems, Priest's inclosure schema (in which a construction both belongs to and exceeds the totality it presupposes) and Derrida's \emph{diff\'erance}, in which any apparent fixity depends on a distinction whose own stability is not preserved across successive operations. We examine these connections below. 

\subsection{Class $\mathbf{A}$ and diagonalisation}
For a class $\mathbf{A}$ system the action of $\hat U$ on $\hat D$ makes the propagation theorem do at system scale what Proposition~\ref{prop:liar} does at sentence scale. The supplement that would sustain the sign and signified separation lies in $\hat D$, with $\hat D\notin\operatorname{Comm}(\hat U)$, and Theorem~\ref{thm:propagation} carries this to $[\hat U,\hat R]\neq0$ outside the cancellation or annihilation cases. The self-reference can no longer stabilise its own referent under $\hat U$, now for the operator by which a whole system identifies its own features rather than for the demonstrative of a single sentence. Applying $\hat R$ then $\hat U$ need not agree with applying $\hat U$ then $\hat R$, and $[\hat U,\hat R]\neq0$ is exactly this disagreement: a self-description formed within the system's own resources fails to preserve its relation to the description $\hat U$ would have carried forward. The following corollary codifies this.
\begin{corollary}[Deconstructive diagonal]\label{cor:diag}
Let $(\hat U,\hat D,\hat R)$ be of class $\mathbf{A}$ and let $|s\rangle\in\mathcal{V}_S$. Let $\rho_0=\hat R|s\rangle$ be the self-description of $|s\rangle$ and $\rho_1=\hat R\,\hat U|s\rangle$ the self-description the system forms of its own modified state $\hat U|s\rangle$. The in-system continuation of $\rho_0$, namely the image $\hat U\rho_0$ that any class $\mathbf{A}$ operation extending $\rho_0$ would produce, satisfies
\[
\rho_1-\hat U\rho_0=-[\hat U,\hat R]\,|s\rangle .
\]
Hence $\rho_1\neq\hat U\rho_0$ for every $|s\rangle\notin\ker([\hat U,\hat R])$. The self-description $\rho_1$ the system forms of its own modified state therefore lies in the self-description space $\hat R(\mathcal{V}_S)$ yet differs from the continuation of $\rho_0$ available within the system, which is the diagonal property.
\end{corollary}

\begin{proof}
By definition $\rho_1=\hat R\hat U|s\rangle$ and $\hat U\rho_0=\hat U\hat R|s\rangle$, so
\[
\rho_1-\hat U\rho_0=(\hat R\hat U-\hat U\hat R)|s\rangle=[\hat R,\hat U]|s\rangle=-[\hat U,\hat R]|s\rangle,
\]
which is the displayed identity and holds for every $|s\rangle$. It remains to show $[\hat U,\hat R]\neq0$. Since the triple is of class $\mathbf{A}$, $\hat R=F(\hat D)$ for a finite composite $F$ built from $\hat D$ and operators in $\operatorname{Comm}(\hat U)$, and $[\hat U,\hat D]\neq0$. By Theorem~\ref{thm:propagation} $[\hat U,\hat R]=\sum_j A_j\,[\hat U,\hat D]\,B_j$ with $A_j,B_j\in\mathcal{A}$ fixed by the factorisation of $F$. This vanishes only if the sum cancels or annihilates, the exceptional cases of Theorem~\ref{thm:propagation}. Outside those cases $[\hat U,\hat R]\neq0$, so $\ker([\hat U,\hat R])$ is a proper subspace of $\mathcal{V}$, and $\rho_1-\hat U\rho_0=-[\hat U,\hat R]|s\rangle\neq0$ for every $|s\rangle\in\mathcal{V}_S\setminus\ker([\hat U,\hat R])$. For such $|s\rangle$ both diagonal clauses hold. Closure: $\rho_1=\hat R\hat U|s\rangle\in\hat R(\mathcal{V}_S)$, since $\hat U|s\rangle\in\mathcal{V}_S$ by invariance and $\hat R(\mathcal{V}_S)\subseteq\mathcal{V}_S$. Transcendence: $\rho_1\neq\hat U\rho_0$, so $\rho_1$ is not the continuation of $\rho_0$ available within the system. \qed
\end{proof}
The corollary exhibits the diagonal directly: the self-description $\rho_1$ the system forms of its own modified state lies in the self-description space $\hat R(\mathcal{V}_S)$, yet differs from $\hat U\rho_0$, the continuation of $\rho_0$ the system itself would form. As we show below, its structural form is the inclosure diagonal familiar from the liar \cite{kripke1975outline}, Russell's paradox, and G\"odel's first incompleteness theorem \cite{godel1931formally,boolos1993logic}, and Proposition~\ref{prop:inclosure} below renders the two inclosure clauses precise for this construction. A closely related pattern appears in Fallenstein and Soares \cite{fallenstein2014problems}, where an agent whose difference operator is a formal proof system fails to certify the behaviour of its own successors via L\"ob's theorem, producing either the L\"obian obstacle or the procrastination paradox. Tiling-agent constructions \cite{yudkowsky2013tiling} also explore similar dilemmas from an architectural perspective. Corollary~\ref{cor:diag} generalises the pattern: it holds for any class $\mathbf{A}$ system, whether $\hat D$ is a formal proof system, an empirical benchmark suite, or any other contrast-based operator. In Derrida's terms $\rho_1$ is the trace of the trace \cite{derrida1982differance}: a self-description that is itself only the deferred mark of a prior self-description, never coinciding with what it succeeds.

\section{Inclosure, Diff\'erance, and Scope}\label{sec:inclosure}
\subsection{Inclosure}\label{sec:inclosure-sub}
Finally, we map our results above to the inclosure schema and diff\'erance. To see the relationship with our results above, consider the typical inclosure schema \cite{priest1994derrida,priest2002beyond,priest2006doubt} as specified by four elements: a totality $\Omega$; a property $\phi$ such that $\Omega=\{x:\phi(x)\}$; a diagonaliser $\sigma$ defined on subsets of $\Omega$; and a contradictory element $\sigma(X)\in\Omega$ belonging to the totality (Closure) while falling outside any particular totalisation $X\subseteq\Omega$ (Transcendence). The schema unifies the liar, Russell's paradox, and G\"odel's first incompleteness theorem \cite{kripke1975outline,godel1931formally,boolos1993logic}: in each case the diagonaliser exceeds any attempt to totalise the property it diagonalises against. Class $\mathbf{A}$ realises the schema on its own operator algebra.

\begin{proposition}[Computational inclosure]\label{prop:inclosure}
Let $(\hat U,\hat D,\hat R)$ be of class $\mathbf{A}$, with admissible configurations $\mathcal{V}_S\subseteq\mathcal{V}$ and $\hat U(\mathcal{V}_S)\subseteq\mathcal{V}_S$. Let $\mathcal{R}_S:=\hat R(\mathcal{V}_S)$ be the space of admissible self-descriptions, and $q:\mathcal{R}_S\to\mathcal{V}_S$ a representative choice with $\hat Rq(\rho)=\rho$ for each $\rho\in\mathcal{R}_S$. The quadruple
\[
\Omega=\mathcal{R}_S,\quad
\phi(\rho)\iff \rho\in\mathcal{R}_S,\quad
\sigma_q(\rho)=\hat R\hat Uq(\rho)
\]
realises Priestian Closure and Transcendence whenever $q(\rho)\notin\ker([\hat U,\hat R])$.
\end{proposition}

\begin{proof}
(1) Closure. Since $q(\rho)\in\mathcal{V}_S$ and $\hat U(\mathcal{V}_S)\subseteq\mathcal{V}_S$, $\hat Uq(\rho)\in\mathcal{V}_S$, so $\sigma_q(\rho)=\hat R\hat Uq(\rho)\in\hat R(\mathcal{V}_S)=\mathcal{R}_S=\Omega$, and $\phi(\sigma_q(\rho))$ holds.
(2) Transcendence. Let $|s\rangle=q(\rho)$, so $\rho=\hat R|s\rangle$. By Corollary~\ref{cor:diag}, $\hat R\hat U|s\rangle-\hat U\hat R|s\rangle=-[\hat U,\hat R]|s\rangle$. If $|s\rangle\notin\ker([\hat U,\hat R])$ then $\sigma_q(\rho)=\hat R\hat U|s\rangle\neq\hat U\rho$. The totalisation is the in-system continuation $\{\hat U\rho\}$ reachable from $\rho$ by $\hat U$, and $\sigma_q(\rho)$ lies outside it. \qed
\end{proof}
A similar construction underlies the liar. At sentence scale $\Omega$ is a sentence-evaluation space and $\sigma$ is re-application of the sign operator $\hat T$. At system scale $\Omega$ is the admissible self-description space $\mathcal{R}_S$ and $\sigma_q$ is re-application of $\hat R$ after $\hat U$, modulo the representative choice $q$. The same schema is instantiated at two scales of self-reference.

\subsection{Diff\'erance}\label{sec:differance-sub}
Priest \cite{priest1994derrida} and Livingston \cite{livingston2010formalizing,livingston2012politics} situate \emph{diff\'erance} as inclosure at the level of philosophical logic. Derrida's \emph{diff\'erance} \cite{derrida1982differance,derrida1976grammatology,derrida1988limited} names the joint operation of differential individuation and temporal deferral in the constitution of identity. The two are inseparable, and their inseparability is constitutive: identity is neither given by unmediated self-presence (a point Derrida develops against Husserl \cite{husserl1973experience}) nor secured by a stable totality of distinctions, but borne by the trace, whose differential content and temporal deferral condition each other. In the operator formalism, $\hat D$ expresses differential individuation and $\hat U$ expresses deferral. Their non-commutation $[\hat U,\hat D]\neq0$ (Theorem~\ref{thm:propagation}) is the failure of stable totalisation that \emph{diff\'erance} identifies as constitutive, and the diagonal of Corollary~\ref{cor:diag} is the trace that cannot be stabilised under its own iteration. Class $\mathbf{A}$'s condition~(iii) is iterability imposed as an architectural requirement: the self-model must cite its features across changes in $\hat D$ or it is not a self-model. Condition~(ii) makes that requirement unsatisfiable without a supplement, since $\hat D$, the operator sustaining iterability, is itself subject to the deferral it was to underwrite.

The two motifs Derrida draws from iterability receive operator-algebraic form. \emph{Grafting} \cite{derrida1988limited} - a mark carried into a context where a prior distinction is erased - is the growth of $\ker\hat U$: configurations distinct before $\hat U$ become identified after. The \emph{supplementary logic} \cite{derrida1976grammatology} - a later context revealing a distinction an earlier one did not register - is the converse: $\hat U$ separates configurations its domain had identified, so distinctions absent before $\hat U$ are marked after. Both arise for any non-trivial $\hat U$ acting on $\hat D$, and either suffices for the non-commutation driving the propagation theorem.

\section{Discussion and conclusion}\label{sec:counter}
We have argued that self-modification, routinely stipulated as constitutive of artificial superintelligence, requires a structure outside the modifying operation - a \emph{supplement} in Derrida's sense - and that when the scope of self-modification extends to this supplement, the classical self-referential construction it sustained collapses. The operator formalism of Sections~\ref{sec:operator}--\ref{sec:setup} makes this structural: the supplement is identified with the commutant of the update operator $\hat U$, a unifying projector is an idempotent in that commutant, and strong self-modification is the regime in which no such projector exists unconditionally. The expansion theorem shows non-commutation between update and discrimination propagates to any self-representation built from the difference operator; the Jacobi corollary restricts unifying supplements to projectors commuting with $[\hat U,\hat D]$; and the liar paradox is the rank-one instance of the same commutator collapse, avoided classically by a linguistic supplement. Class $\mathbf{A}$ executes the collapse at system scale through the diagonal of Corollary~\ref{cor:diag}. In this sense, class $\mathbf{A}$ systems are the liar writ large. The consequences of heightened self-modification for superintelligent system identity are noteworthy. Further research may examine how identity claims about such systems cannot be read as unconditional predicates of the system itself. They are relative, contingent on which supplement a frame fixes as unifying. This is analogously the operator-algebraic form of Derrida's claim that there is no master-name \cite{derrida1982differance}: no single $\Pi$ fixes identity unconditionally, and each unifying $\Pi$ a frame selects is displaced at the next application of $\hat U$. For superintelligent systems exhibiting strong self-modification, this means that the certainty upon which its values, objectives, commitments - and even its own sense of self-identity and its own boundaries - rely are potentially problematised by its capacity for strong self-modification.

\bibliographystyle{splncs04}
\bibliography{refs2}

\end{document}